%% file: main.tex
\documentclass[letterpaper, 10 pt, conference]{ieeeconf}
\IEEEoverridecommandlockouts

\font\stixfrak=stix-mathfrak at 10pt
-mathcal at 10pt

\usepackage{dutchcal}
\usepackage{amsmath}
\usepackage{amsfonts}
\usepackage{amssymb}
\usepackage{amsxtra}
\usepackage{amsthm} % Unnecessary in LNCS
\usepackage{leftidx}
\usepackage{url}
\usepackage{graphicx,color}
\usepackage{float}
\usepackage{stfloats}
\usepackage{dsfont}
\usepackage{mathrsfs}
\usepackage{array}
\usepackage{rotating}
\usepackage{calc}
\usepackage{tikz-cd}
\usepackage{dsfont}
\usepackage{stmaryrd}
\usepackage{cancel}
\usepackage{accents}
\let\labelindent\relax
\usepackage{enumitem}
\usepackage[linesnumbered,lined,boxed,commentsnumbered]{algorithm2e}
\SetAlFnt{\footnotesize}
\SetAlCapFnt{\rmfamily\footnotesize}
\usepackage[colorlinks=true,urlcolor=black,linkcolor=blue,citecolor=blue]{hyperref}
\usepackage[capitalize,nameinlink]{cleveref}
\usepackage{comment}
\usepackage{soul}

\usepackage{mathrsfs}

\usepackage{multicol}
\usepackage{subcaption}

\makeatletter
\newcommand\Func[2]{%
    \textbf{function} #1%\{%
    \algocf@group{#2}%
}
\makeatother

\makeatletter
\newcommand\Forr[2]{%
    \textbf{for} #1 \textbf{do}%
    \algocf@group{#2}%
}
\makeatother

\makeatletter
\newcommand\Blnk[2]{%
    \hspace{10pt}#1%\{%
    \algocf@group{#2}%
    \textbf{end}%\}\;%
}
\makeatother

\usepackage{color} 

\usepackage[normalem]{ulem}

\makeatletter
\newcommand{\removelatexerror}{\let\@latex@error\@gobble}
\makeatother

\makeatletter
\def\@endtheorem{\endtrivlist}
\makeatother

\newtheorem{theorem}{Theorem}
\newtheorem{definition}{Definition}
\newtheorem{proposition}{Proposition}
\newtheorem{remark}{Remark}

\newtheorem{problem}{Problem}
\date{\today}

\newcommand{\nn}{{\mathscr{N}\negthickspace\negthickspace\negthinspace\mathscr{N}}\negthinspace}
\newcommand{\ou}{%
  \mathrel{%
    \vcenter{\offinterlineskip
      \ialign{##\cr$<$\cr\noalign{\kern-1.5pt}$>$\cr}%
    }%
  }%
}%

\newcommand{\overbar}[1]{\mkern 3.0mu\overline{\mkern-2.5mu#1\mkern-2.0mu}\mkern 2.0mu}

\newcommand{\lblkbrbrack}{\negthinspace\text{{\stixfrak\char"36}}\normalfont}
\newcommand{\rblkbrbrack}{\text{{\stixfrak\char"37}}\normalfont}
\newcommand{\tllparams}[2]{\Xi_{\scriptscriptstyle #1,#2}}
\newcommand{\multitllparams}[3]{\Xi^{(#3)}_{\scriptscriptstyle #1,#2}}

\newcommand{\subarg}[1]{\lblkbrbrack #1 \rblkbrbrack}

\renewcommand{\baselinestretch}{0.91}

\begin{document}

% begin nowrap
\title{
\LARGE{\bf Bounding the Complexity of Formally Verifying Neural Networks: \\A Geometric Approach
}
} %
\author{James Ferlez\textsuperscript{$*$} and Yasser Shoukry\textsuperscript{$*$}
% \thanks{
% \textsuperscript{$\dagger$} Equally contributing first authors.
% }
\thanks{
\textsuperscript{$*$}Department of Electrical Engineering and Computer Science, University of California, Irvine
\texttt{\{jferlez,yshoukry\}@uci.edu}
} %
%\thanks{\textsuperscript{$\ddagger$}Affiliation 2
%\texttt{\{email\}@email.edu}
%}
\thanks{This  work  was  partially  sponsored  by  the  NSF  awards \#CNS-2002405 and \#CNS-2013824.}
} %
%\thanks{
%\textsuperscript{$\dagger$} Equally contributing first authors.
%}
% end nowrap

\maketitle

\begin{abstract}
	%In this paper, we consider the computational complexity of formally 
	%verifying the input/output behavior of Rectified Linear Unit (ReLU) Neural 
	%Networks (NNs), where verification entails determining whether the output 
	%of a NN lies in a specific convex polytopic region whenever its input lies 
	%in a specific polytopic region. 
	In this paper, we consider the computational complexity of formally 
	verifying the behavior of Rectified Linear Unit (ReLU) Neural Networks 
	(NNs), where verification entails determining whether the NN satisfies 
	convex polytopic specifications. Specifically, we show that for two 
	different NN architectures -- shallow NNs and Two-Level Lattice (TLL) NNs 
	-- the verification problem with (convex) polytopic constraints is 
	\emph{polynomial} in the number of neurons in the NN to be verified, when 
	all other aspects of the verification problem held fixed. We achieve these 
	complexity results by exhibiting explicit (but similar) verification 
	algorithms for each type of architecture. Both algorithms efficiently
	%(or almost directly)
	translate the NN parameters into a partitioning of the NN's input space by 
	means of \emph{hyperplanes}; this has the effect of partitioning the 
	original verification problem into polynomially many sub-verification 
	problems derived from the geometry of the neurons. 
	% ; on 
	% each of which the NN in question an affine function.
	% Indeed, these 
	% partitionings have two further important properties. First, the number of 
	% these hyperplanes is polynomially related to the number of neurons, and 
	% hence so is the number of sub-verification problems.
	%(via an fundamental result on hyperplanes). 
	We show that these sub-problems may be chosen so that the NN is purely 
	affine within each, and hence each sub-problem is solvable in polynomial 
	time by means of a Linear Program (LP). Thus, a polynomial-time algorithm 
	for the original verification problem can be obtained using known 
	algorithms for enumerating the regions in a hyperplane arrangement.
	% is only necessary to 
	% enumerate these subproblems in polynomial time, and there are known 
	% algorithms to accomplish this.
	%(so they can be solved exhaustively). 
	% To achieve this final step, we also contribute a novel algorithm to 
	% enumerate the regions in a hyperplane arrangement in polynomial time; our 
	% algorithm is based on a poset ordering of the regions for which poset 
	% successors are polynomially easy to compute. Taken together, the runtime of 
	% either verification algorithm has the following form: polynomially many 
	% subproblems times polynomial time to enumerate all of the subproblems times 
	% polynomial time to solve each subproblem.
	Finally, we adapt our proposed algorithms to the verification of dynamical 
	systems, specifically when these NN architectures are used as 
	state-feedback controllers for LTI systems. We further evaluate the 
	viability of this approach numerically.
	% we implement a version of our proposed verification algorithm for 
	% TLL networks, and demonstrate that it can be used in practice to verify 
	% modest-sized TLL NNs. For such NNs, we used our algorithm to verify both 
	% static input/output properties and closed-loop specifications when said NN 
	% was used as a controller for an LTI system.
\end{abstract}

\input{intro.tex}

\input{problem.tex}

\input{framework}
\input{verifier}

\input{experiments}

% \input{discussion}

% \input{main_algorithm}

% \input{mpc_regions}

% \input{unique_order_approx}

% \input{discussion}

% \input{numerical_results}

% \input{appendix}

\bibliographystyle{plain} %
\bibliography{mybib}

\end{document}

%% file: intro.tex
% !TEX root = ./main.tex

\section{Introduction} % (fold)
\label{sec:introduction}
Neural Networks (NNs) are increasingly
% used in modern cyber-physical systems, 
% where they are often
used as feedback controllers in safety-critical cyber-physical systems, so
% networks are usually trained from data in an end-to-end learning paradigm. Such 
% data-trained controllers often perform quite well, but because they are 
% obtained by an \emph{implicit} design methodology, they do not come with 
% crucial safety guarantees per se. 
% %Indeed, the fragility of data-trained NNs is 
% % now well known: the performance of such NNs can often be degraded even with 
% % relatively small malicious, adversarial inputs~\cite{szegedy2013intriguing, 
% %goodfellow2014explaining, kurakin2016adversarial, song2018physical}.  
% Thus, algorithms that can independently verify the safety of a NN (controller) 
% are crucial to keep pace with the momentum of NN controller adoption in 
% safety-critical applications.
algorithms that can \emph{verify} the safety of such controllers are of crucial 
importance. However, despite the importance of NN verification algorithms, 
relatively little attention has been paid to an analysis of their computational 
complexity. Such considerations are especially relevant in the verification of 
controllers, since a verifier may be invoked many times to verify a controller 
in closed loop.

On the one hand, it is known that the satisfiability of any 3-SAT formula can 
be encoded as a NN verification problem, but this result requires its variables 
to be in correspondence with the input dimensions to the network 
\cite{KatzReluplexEfficientSMT2017a}. This means the complexity of verifying a 
NN depends unfavorably on the dimension of its input space. On the other hand, 
this result doesn't address the relative difficulty of verifying a NN with a 
\emph{fixed} input dimension but an increasing number of neurons. The only 
results in this vein exhibit networks for which the number of affine regions 
grows exponentially in the number of neurons in the network -- see e.g. 
\cite{MontufarNumberLinearRegions2014}. However, these merely suggest that the 
verification problem is still ``hard'' in the number of neurons in the network 
(input and output dimensions are fixed). There are not, to our knowledge, any 
concrete complexity results that address this second question.

In this paper, we we prove two such concrete complexity results that explicitly 
describe the computational complexity of verifying a NN as a function of its 
size. In particular, we prove that the complexity of verifying either a shallow 
NN or a Two-Level Lattice NN \cite{FerlezAReNAssuredReLU2020} grows only 
\emph{polynomially} in the number of neurons in the network to be verified, all 
other aspects of the verification problem held fixed. These results appear in 
Section \ref{sec:framework} as Theorem \ref{thm:main_shallow_nn_thm} and 
Theorem \ref{thm:main_tll_nn_thm} for shallow NNs and TLL NNs, respectively. 
Our proofs for both of these complexity results are existential: that is we 
propose one concrete verification algorithm for each architecture. 
By their mere existence, the complexity results we prove herein demonstrate 
that the NN verification problem is not per se a ``hard'' problem as a function 
of the size of the NN to be verified.  Moreover, although our results show that 
the complexity of verifying a shallow or a TLL NN scales polynomially with its 
size, our complexity claims \emph{do} scale exponentially in the dimension of 
the input to the NN. Thus, our results do not contradict the known 
%``hardness'' of the 
%verification problem as a function of the NN's input dimension -- i.e. the 
results in \cite{KatzReluplexEfficientSMT2017a}. One further observation is in 
order: while our results do speak directly to the complexity of the 
verification problem as a function of the number of neurons, they \emph{do not} 
address the complexity of the verification problem in terms of the 
\emph{expressivity} of a particular network size; see Section 
\ref{sec:framework}.
% . We consider this aspect in 
% more detail in Section \ref{sec:framework}.

Moreover, the nature of our proposed algorithms means \emph{they have direct 
applicability to verifying such NNs when they are used as feedback 
controllers.} In particular, they verify a NN by dividing its input space into 
regions on which the NN is affine; in this context, verifying an input/output 
property requires one Linear Program (LP) on each such region, but such an LP 
can easily be extended to verify certain discrete-time dynamical properties for 
LTI systems. That is our algorithm can verify whether the \emph{next state} 
resulting from a state-feedback NN controller lies in a particular polytopic 
set (e.g. forward invariance of a (polytopic) set of states).% in 
%closed-loop. 

We conclude this paper with a set of experimental results that validate the 
claims we have made about our proposed TLL verifier. First, we show that our 
implementation does in fact scale polynomially. And second, we show that it can 
be adapted to verify the forward invariance of a polytopic set of states (on an 
LTI system with state feedback TLL controller).

% In this paper, we are chiefly interested in understanding the computational 
% complexity of the verification task as a function of the number of neurons in 
% the network to be verified. Moreover, to help demarcate the fundamental 
% complexity of the problem, we will further concern ourselves with algorithms 
% that have \emph{polynomial} complexity in the number of neurons. Thus, 

\noindent \textbf{Related work:} %The idea of %
Most of the work on NN verification has focused on obtaining practical 
algorithms rather than theoretical complexity results, although many have 
noticed empirically that there is a significant complexity associated with the 
input dimension; \cite{KatzReluplexEfficientSMT2017a} is a notable exception, 
since it also included an NP-completeness result based on the 3-SAT encoding 
mentioned above. Other examples of pragmatic NN verification approaches 
include: (i) SMT-based methods;
%, which encode the 
%problem into a Satisfiability Modulo Theory problem~
% \cite{katz2019marabou, KatzReluplexEfficientSMT2017a, ehlers2017formal}; %
(ii) %
MILP-based solvers;
%, which directly encode the verification problem as a Mixed 
%Integer Linear Program~
% \cite{lomuscio2017approach, tjeng2017evaluating, bastani2016measuring, 
% bunel2020branch, fischetti2018deep, anderson2020strong, cheng2017maximum}; 
(iii) Reachability based methods; and
%, which perform layer-by-layer reachability 
%analysis to compute the reachable set~
% \cite{xiang2017reachable, xiang2018output, gehr2018ai2, wang2018formal, 
% tran2020nnv, ivanov2019verisig, fazlyab2019efficient}; %and 
(iv) convex relaxations methods. A good survey of these methods an be found in 
\cite{LiuAlgorithmsVerifyingDeep2019}.
% ~\cite{wang2018efficient, dvijotham2018dual, 
% wong2017provable}.
%In general, (i), (ii) and (iii) suffer from poor scalability.
% On the other hand, convex relaxation methods depend heavily on 
% pruning the search space of indeterminate neuron activations; thus, they 
% generally depend on obtaining good approximate bounds for each of the neurons 
% in order to reduce the search space (the exact bounds are computationally 
% intensive to compute \cite{dutta2017output}). These methods are most similar to 
% PeregriNN: for example, \cite{wang2018formal,bunel2020branch, royo2019fast} 
% recursively refine the problem using input splitting, and 
% \cite{wang2018efficient} does so via neuron splitting. Other search and 
% optimization methods include: Planet \cite{ehlers2017formal}, which combines a 
% relaxed convex optimization problem with a SAT solver to search over neurons' 
% phases; and Marabou~\cite{katz2019marabou}, which uses a modified simplex 
% algorithm to handle non-convex ReLU activations.
By contrast, a number of works have focused on the computational complexity of 
various other verification-related questions for NNs 
(\cite{KatzReluplexEfficientSMT2017a} is the exception in that it expressly 
considers the verification problem itself). Some NN-related complexity results 
include: the minimum adversarial disturbance to a NN is NP hard  
\cite{WengFastComputationCertified2018}; computing the Lipschitz constant of a 
NN is NP hard \cite{VirmauxLipschitzRegularityDeep2018}; reachability analysis 
is NP hard \cite{RuanReachabilityAnalysisDeep2018a}, 
\cite{TranStarBasedReachabilityAnalysis2019}.

% section introduction (end)

%% file: problem.tex
% !TEX root = ./main.tex

\section{Preliminaries} % (fold)
\label{sec:preliminaries}

\subsection{Notation} % (fold)
\label{sub:notation}
We will denote the real numbers by $\mathbb{R}$. For an $(n \times m)$ matrix 
(or vector), $A$, we will use the notation $\llbracket A \rrbracket_{i,j}$ to 
denote the element in the $i^\text{th}$ row and $j^\text{th}$ column of $A$. 
Analogously, the notation $\llbracket A \rrbracket_{i,\cdot}$ will denote the 
$i^\text{th}$ row of $A$, and $\llbracket A \rrbracket_{\cdot, j}$ will denote 
the $j^\text{th}$ column of $A$; when $A$ is a vector instead of a matrix, both 
notations will return a scalar corresponding to the corresponding element in 
the vector. Let $\mathbf{0}_{n,m}$ be an $(n \times m)$ matrix of zeros. We 
will use bold parenthesis $\;\subarg{ \cdot }$ to delineate the arguments to a 
function that \emph{returns a function}. For example, given an $(m \times n)$ 
matrix, $W$, (possibly with $m=1$) and an $(m \times 1)$ dimensional vector, 
$b$, we define the linear function: $ \mathscr{L}^i \subarg{ W, b } : x \mapsto 
\llbracket W x + b \rrbracket_i$ (that is $\mathscr{L}^i \subarg{ W, b }$ is 
itself a function). We also use the functions $\mathtt{First}$ and 
$\mathtt{Last}$ to return the first and last elements of an ordered list (or by 
overloading, a vector in $\mathbb{R}^n$). The function $\mathtt{Concat}$ 
concatenates two ordered lists, or by overloading, concatenates two vectors in 
$\mathbb{R}^n$ and $\mathbb{R}^m$ along their (common) nontrivial dimension to 
get a third vector in $\mathbb{R}^{n+m}$. We will use an over-bar to indicate 
(topological) closure of a set: i.e. $\overbar{A}$ is the closure of $A$. 
Finally, $B(x;\delta)$ denotes an open Euclidean ball centered at $x$ with 
radius $\delta$.
% subsection notation (end)

\subsection{Neural Networks} % (fold)
\label{sub:neural_networks}
We will exclusively consider Rectified Linear Unit Neural Networks (ReLU NNs). 
A $K$-layer ReLU NN is specified by composing $K$ \emph{layer} functions. We 
allow two kinds of layers: linear and nonlinear. A \emph{nonlinear} layer with 
$\mathfrak{i}$ inputs and $\mathfrak{o}$ outputs is specified by a 
$(\mathfrak{o} \times \mathfrak{i} )$ real-valued matrix of \emph{weights}, 
$W$, and a $(\mathfrak{o} \times 1)$ real-valued matrix of \emph{biases}, $b$: 
% as follows:
\begin{equation}
	L_{\theta} : \mathbb{R}^{\mathfrak{i}} \rightarrow \mathbb{R}^{\mathfrak{o}},  \qquad
	    L_{\theta} :  z \mapsto \max\{ W z + b, 0 \}
\end{equation}
where the $\max$ function is taken element-wise, and $\theta \triangleq (W,b)$. 
A \emph{linear} layer is the same as a nonlinear layer, except it omits the 
nonlinearity $\max\{\cdot , 0\}$ in its layer function; a linear layer will be 
indicated with a superscript \emph{lin} e.g. $L^\text{lin}_{\theta}$

Thus, a $K$-layer ReLU NN function as above is specified by $K$ layer functions 
$\{L_{\theta^{(i)}} : i = 1, \dots, K\}$ whose input and output dimensions are 
\emph{composable}: that is they satisfy $\mathfrak{i}_{i} = \mathfrak{o}_{i-1}: 
i = 2, \dots, K$. \textbf{We adopt the convention that the final layer is 
always a linear layer}, but other layers may be linear or not.
% Specifically:
% \begin{equation}
% 	\nn(x) = (L_{\theta^{|K}}^\text{lin} \circ L_{\theta^{|K-1}} \circ \dots \circ L_{\theta^{|1}})(x).
% \end{equation}
To make the dependence on parameters explicit, we will index a ReLU function 
$\nn$ by a \emph{list of matrices} $\Theta \triangleq$ $( \theta^{|1},$ 
$\dots,$ $\theta^{|K} );$\footnote{That is $\Theta$ is not the concatenation of 
the $\theta^{(i)}$ into a single large matrix, so it preserves information 
about the sizes of the constituent $\theta^{(i)}$.} in this respect, we will 
often use $\nn_\Theta$ and $\Theta$ interchangeably when no confusion will 
result.
% It is common to omit the $\max$ function from the final 
% layer, and this will be our convention hereafter.

The number of layers and the \emph{dimensions} of the associated matrices 
$\theta^{|i} = (\; W^{|i}, b^{|i}\; )$ specifies the \emph{architecture} of the 
ReLU NN. Therefore, we will use:
\begin{equation}
	\text{Arch}(\Theta) \triangleq ( (n,\mathfrak{o}_{1}), (\mathfrak{i}_{2},\mathfrak{o}_{2}), \ldots, %(\mathfrak{i}_{K-1},\mathfrak{o}_{K-1}), 
	(\mathfrak{i}_{K}, m))
\end{equation}
to denote the architecture of the ReLU NN $\nn_{\Theta}$.
% Note that our 
% definition is quite general since it allows the layers to be of different 
% sizes, as long as $\mathfrak{o}_{i-1} = \mathfrak{i}_{i}$ for $i = 2, \dots, K$.

\begin{definition}[Shallow NN]\label{def:shallow_nn}
	A \textbf{shallow NN} is two-layer NN whose first layer is nonlinear and 
	whose second is linear.
\end{definition}

% subsection rectified_linear_unit_neural_networks (end)

\subsection{Special NN Operations} % (fold)
\label{sub:special_nn_operations}
% Here we define two different mechanisms for combining a pair of NNs in order to 
% obtain a third.
\begin{definition}[Sequential (Functional) Composition]
\label{def:functional_composition}
	Let $\nn_{\Theta_{\scriptscriptstyle 1}}$ and 
	$\nn_{\Theta_{\scriptscriptstyle 2}}$ be two NNs where 
	$\mathtt{Last}(\text{Arch}(\Theta_1)) = (\mathfrak{i}, \mathfrak{c})$ and 
	$\mathtt{First}(\text{Arch}(\Theta_2)) =  (\mathfrak{c}, \mathfrak{o})$ for 
	some nonnegative integers $\mathfrak{i}$, $\mathfrak{o}$ and 
	$\mathfrak{c}$. Then the \textbf{sequential (or functional) composition} of 
	$\nn_{\Theta_{\scriptscriptstyle 1}}$ and $\nn_{\Theta_{\scriptscriptstyle 
	2}}$, i.e. $\nn_{\Theta_{\scriptscriptstyle 1}} \circ 
	\nn_{\Theta_{\scriptscriptstyle 2}}$, is a well defined NN, and can be 
	represented by the parameter list $\Theta_{1} \circ \Theta_{2} \triangleq 
	\mathtt{Concat}(\Theta_1, \Theta_2)$.
\end{definition}
\begin{definition}
	\label{def:parallel_composition}
	Let $\nn_{\Theta_{\scriptscriptstyle 1}}$ and 
	$\nn_{\Theta_{\scriptscriptstyle 2}}$ be two $K$-layer NNs with parameter 
	lists:
	\begin{equation}
		\Theta_i = ((W^{\scriptscriptstyle |1}_i, b^{\scriptscriptstyle |1}_i), \dots, (W^{\scriptscriptstyle |K}_i, b^{\scriptscriptstyle |K}_i)), \quad i = 1,2.
	\end{equation}
	Then the \textbf{parallel composition} of $\nn_{\Theta_{\scriptscriptstyle 
	1}}$ and $\nn_{\Theta_{\scriptscriptstyle 2}}$ is a NN given by the 
	parameter list
	\begin{equation}
		\Theta_{1} \parallel \Theta_{2} \triangleq \big(\negthinspace
			\left(
				\negthinspace
				\left[
					\begin{smallmatrix}
						W^{\scriptscriptstyle |1}_1 \\
						W^{\scriptscriptstyle |1}_2
					\end{smallmatrix}
				\right],
				\left[
					\begin{smallmatrix}
						b^{\scriptscriptstyle |1}_1 \\
						b^{\scriptscriptstyle |1}_2
					\end{smallmatrix}
				\right]
				\negthinspace
			\right),
			{\scriptstyle \dots},
			\left(
				\negthinspace
				\left[
					\begin{smallmatrix}
						W^{\scriptscriptstyle |K}_1 \\
						W^{\scriptscriptstyle |K}_2
					\end{smallmatrix}
				\right],
				\left[
					\begin{smallmatrix}
						b^{\scriptscriptstyle |K}_1 \\
						b^{\scriptscriptstyle |K}_2
					\end{smallmatrix}
				\right]
				\negthinspace
			\right)
		\negthinspace\big).
	\end{equation}
	That is $\Theta_{1} \negthickspace \parallel \negthickspace \Theta_{2}$ 
	accepts an input of the same size as (both) $\Theta_1$ and $\Theta_2$, but 
	has as many outputs as $\Theta_1$ and $\Theta_2$ combined.
\end{definition}

\begin{definition}[$n$-element $\min$/$\max$ NNs]
	\label{def:n-element_minmax_NN}
	An $n$\textbf{-element $\min$ network} is denoted by the parameter list 
	$\Theta_{\min_n}$. $\nn\subarg{\Theta_{\min_n}}: \mathbb{R}^n \rightarrow 
	\mathbb{R}$ such that $\nn\subarg{\Theta_{\min_n}}(x)$ is the the minimum 
	from among the components of $x$ (i.e. minimum according to the usual order 
	relation $<$ on $\mathbb{R}$). An $n$\textbf{-element $\max$ network} is 
	denoted by $\Theta_{\max_n}$, and functions analogously. These networks are 
	described in \cite{FerlezAReNAssuredReLU2020}.
	% These networks are 
	% defined in Appendix \ref{sec:appendix}.
\end{definition}

\subsection{Two-Level-Lattice (TLL) Neural Networks} % (fold)
\label{sub:two_layer_lattice_neural_networks}
In this paper, we will be especially concerned with ReLU NNs that have the 
Two-Level Lattice (TLL) architecture, as introduced with the AReN algorithm in 
\cite{FerlezAReNAssuredReLU2020}. We describe the TLL architecture in two 
separate subsections, one for scalar output TLL NNs, and one for multi-output 
TLL NNs. 
\subsubsection{Scalar TLL NNs} % (fold)
\label{ssub:scalar_tll_nns}
From \cite{FerlezAReNAssuredReLU2020}, a scalar-output TLL NN can be described 
as follows.
\begin{definition}[Scalar TLL NN {\cite[Theorem 2]{FerlezAReNAssuredReLU2020}}]
\label{def:scalar_tll}
A NN that maps $\mathbb{R}^n \rightarrow \mathbb{R}$ is said to be \textbf{TLL 
NN of size} $(N,M)$ if the size of its parameter list $\Xi_{\scriptscriptstyle 
N,M}$ can be characterized entirely by integers $N$ and $M$ as follows.
\begin{equation}
	\Xi_{N,M} \negthinspace \triangleq  \negthinspace
		\Theta_{\max_M} \negthinspace\negthinspace
	\circ \negthinspace
		\big(
			(\negthinspace\Theta_{\min_N} \negthinspace \circ \Theta_{S_1}\negthinspace) \negthinspace
			\parallel \negthinspace {\scriptstyle \dots} \negthinspace \parallel \negthinspace
			(\negthinspace\Theta_{\min_N} \negthinspace \circ \Theta_{S_M}\negthinspace)
		\big) \negthinspace
	\circ 
		\Theta_{\ell}
\end{equation}
where

\begin{itemize}
	\item $\Theta_\ell \triangleq ((W_\ell, b_\ell))$;

	\item  each $\Theta_{S_j}$ has the form $\Theta_{S_j} = \big( S_j, 
\mathbf{0}_{M,1} \big)$; and

	\item $S_j = \left[ \begin{smallmatrix} {\llbracket I_N 
		\rrbracket_{\iota_1, 
		\cdot}}\negthickspace\negthickspace\negthickspace^{^{\scriptscriptstyle\text{T}}} 
		& \; \dots \; & {\llbracket I_N \rrbracket_{\iota_N, 
		\cdot}}\negthickspace\negthickspace\negthickspace^{^{\scriptscriptstyle\text{T}}} 
		\end{smallmatrix} \right]^\text{T}$ for some sequence $\iota_k \in \{1, 
		\dots, N\}$ ($I_N$ is the $(N \times N)$ identity matrix). 
\end{itemize}

% In the above,
% \begin{equation}
% 	\Theta_\ell \triangleq ((W_\ell, b_\ell)),
% \end{equation}
% and each $\Theta_{S_j}$, $j = 1, \dots, M$ has the form:
% \begin{equation}
% 	\Theta_{S_j} = \big(
% 		S_j,
% 		\mathbf{0}_{M,1}
% 	\big)
% \end{equation}
% where
% \begin{equation}
% 	S_j = \left[ 
% 		\begin{smallmatrix}
% 			{\llbracket I_N \rrbracket_{\iota_1, \cdot}}\negthickspace\negthickspace\negthickspace^{^{\scriptscriptstyle\text{T}}} & \;
% 			\dots \; &
% 			{\llbracket I_N \rrbracket_{\iota_N, \cdot}}\negthickspace\negthickspace\negthickspace^{^{\scriptscriptstyle\text{T}}}
% 		\end{smallmatrix}
% 	\right]^\text{T}
% \end{equation}
% for some sequence $\iota_k \in \{1, \dots, N\}$ (recall that $I_N$ is the $(N 
% \times N)$ identity matrix).

The linear functions implemented by the mapping $x \mapsto \llbracket W_\ell 
\rrbracket_{i, \cdot} \cdot x + \llbracket b_\ell \rrbracket_{i, \cdot}$ for $i 
= 1, \dots, N$ will be referred to as the \textbf{local linear functions} of 
$\Xi_{N,M}$; we assume for simplicity that these linear functions are unique. 
The matrices $\{ S_j | j = 1, \dots, M\}$ will be referred to as the 
\textbf{selector matrices} of $\Xi_{N,M}$. Each set $s_j \triangleq \{ k \in 
\{1, \dots, N\} | \exists \iota \in \{1, \dots, N\}. \llbracket S_j 
\rrbracket_{\iota,k} = 1 \}$ is said to be the selector set of $S_j$.
\end{definition}

\subsubsection{Multi-output TLL NNs} % (fold)
\label{ssub:multi_output_tll_nns}
% The TLL NN architecture is based is based entirely on $\min$ and $\max$ 
% operations over the real numbers, which in turn depend on the usual total 
% ordering of $\mathbb{R}$ by $<$. Nevertheless, a CPWA function $\mathbb{R}^n 
% \rightarrow \mathbb{R}^m$ is necessarily describable in terms of $m$ 
% real-valued CPWAs over $\mathbb{R}^n$. Thus, the only significant difference 
% between the scalar TLL architecture and the multi-output TLL architecture is an 
% accounting of the number of local linear functions (and to a lesser extent, the 
% unique order regions) of the combined network.

We will define a multi-output TLL NN with range space $\mathbb{R}^m$ using $m$ 
equally sized scalar TLL NNs. That is we denote such a network by 
$\Xi^{(m)}_{N,M}$, with each output-components denoted by $\Xi^{i}_{N,M}$, $i = 
1, \dots, m$.
% This is for two 
% reasons. First, it will make the eventual computational complexity expressions 
% for our algorithms more compact. Second, it will make straightforward the 
% connection to assured architecture designs, such as the one in 
% \cite{FerlezAReNAssuredReLU2020}: if $N$ local linear functions are needed in a 
% NN with $m$ real-valued outputs, then an architecture with $m$ \emph{component} 
% TLL NNs, each with its own $N$ local linear functions, will have enough local 
% linear functions to meet the assurance (a consequence of \cite[Theorem 
% 3]{FerlezAReNAssuredReLU2020}). We will likewise assume that all of the 
% component TLLs have a common number of selector matrices.

% \begin{definition}[Multi-output TLL NN]
% 	\label{def:multi-output_tll}
% 	A ReLU NN that maps $\mathbb{R}^n \rightarrow \mathbb{R}^m$ is said to be 
% 	an $m$\textbf{-output TLL NN of size} $(N,M)$ if its parameters 
% 	$\Xi^{(m)}_{N,M}$ are the parallel composition of $m$ scalar-output TLL 
% 	NNs, each of size $(N,M)$. That is:
% 	\begin{equation}
% 		\Xi^{(m)}_{N,M} \triangleq (\Xi^{1}_{N,M} \parallel \dots \parallel \Xi^{m}_{N,M})
% 	\end{equation}
% 	The subnetworks $\Xi^{\iota}_{N,M}, \iota = 1, \dots, m$ will be referred 
% 	to as the \textbf{component (TLL) networks} of $\Xi^{(m)}_{N,M}$.
% \end{definition}
% subsubsection mutl_output_tll_nns (end)

% subsection two_layer_lattice_neural_networks (end)

\subsection{Hyperplanes and Hyperplane Arrangements} % (fold)
\label{sub:hyperplanes_and_hyperplane_arrangements}
Here  we review notation for hyperplanes and hyperplane arrangements; 
\cite{StanleyIntroductionHyperplaneArrangements} is the main reference for this 
section.%, but the fancier theorems therein aren't needed here.
\begin{definition}[Hyperplanes and Half-spaces]
	Let $\ell : \mathbb{R}^n \rightarrow \mathbb{R}$ be an affine map. Then 
	define:
	\begin{equation}
		H^{q}_{\ell} \triangleq 
		\begin{cases}
			\{x | \ell(x) < 0 \} & q = -1 \\
			\{x | \ell(x) > 0 \} & q = +1 \\
			\{x | \ell(x) = 0 \} & q = 0.
		\end{cases}
	\end{equation}
	We say that $H^{0}_\ell$ is the \textbf{hyperplane defined by} $\ell$ 
	\textbf{in dimension }$n$, and $H^{-1}_\ell$ and $H^{+1}_\ell$ are the 
	\textbf{negative and positive half-spaces defined by} $\ell$, respectively.
\end{definition}

% \begin{definition}[Normal Vector to a Hyperplane]
% 	Let $H^0_\ell$ be a hyperplane. Then let $\perp\negthickspace(H^0_\ell)$ be 
% 	the unit normal vector to $H^0_\ell$ such that for any $x \in H^0_\ell$, 
% 	$(x + \perp\negthickspace(H^0_\ell) ) \in H^{+1}_\ell$.
% \end{definition}

% \begin{definition}[Rank of a set of Hyperplanes]
% 	Let $S = \{ H^0_{\ell_1},$ $\dots, H^0_{\ell_N} \}$ be a set of hyperplanes 
% 	with associated affine functions $\ell_1, \dots, \ell_N$. Then we define:
% 	\begin{equation}
% 		\text{\upshape rank}(S) \triangleq
% 		\text{\upshape rank}(\{\perp\negthickspace(H^0_\ell) | H^0_\ell \in S\}).
% 	\end{equation}
% \end{definition}

\begin{definition}[Hyperplane Arrangement]
	Let $\mathcal{L}$ be a set of affine functions where each $\ell \in 
	\mathcal{L} : \mathbb{R}^n \rightarrow \mathbb{R}$. Then $\{ H^{0}_{\ell} | 
	\ell \in \mathcal{L} \}$ is an \textbf{arrangement of hyperplanes in 
	dimension} $n$.
\end{definition}

\begin{definition}[Region of a Hyperplane Arrangement]
\label{def:hyperplane_region}
	Let $\mathcal{H}$ be an arrangement of $N$ hyperplanes in dimension $n$ 
	defined by a set of affine functions, $\mathcal{L}$. Then a non-empty open 
	subset $R \subseteq \mathbb{R}^n$ is said to be a ($n$\textbf{-dimensional) 
	region of} $\mathcal{H}$ if there is an indexing function $\mathfrak{s} : 
	\mathcal{L} \rightarrow \{-1, 0 ,+1\}$ such that
	% $\text{\upshape rank}(\{ 
	% \perp \negthinspace H^0_\ell | \ell \in \mathcal{L}, \mathfrak{s}(\ell) = 
	% 0\}) = n$ and
	$R = \bigcap_{\ell \in \mathcal{L}} H^{\mathfrak{s}(\ell)}_\ell$ and 
	$B(x,\delta) \subset R$ for some $x$ and $\delta$.
	% Regions will be assumed to be the same as the dimension of the arrangement 
	% unless noted.
	The set of all regions of arrangement $\mathcal{H}$ will be denoted 
	$\mathcal{R}_\mathcal{H}$.
\end{definition}

% \begin{proposition}
% 	A region, $R$, of a hyperplane arrangement is uniquely specified by its 
% 	indexing function, $\mathfrak{s}$.
% \end{proposition}

% \begin{definition}[Face]
% \label{def:region_face}
% 	Let $\mathcal{H}$ be an arrangement of $N$ hyperplanes in dimension $n$ 
% 	defined by affine functions $\mathcal{L}$, and let $R$ be an 
% 	$m$-dimensional region of $\mathcal{H}$ that is specified by the indexing 
% 	function $\mathfrak{s} : \mathcal{L} \rightarrow \{-1, 0, +1\}$. Then a 
% 	closed set $F \subset \mathbb{R}^n$ is an $m^\prime \leq m$ 
% 	\textbf{dimensional face of} $R$ if $F$ is the closure of an 
% 	$m^\prime$-dimensional region contained in $\overbar{R}$. That is there 
% 	exists an indexing function $\mathfrak{s}^\prime : \mathcal{L} \rightarrow 
% 	\{-1, 0 ,+1\}$ such that 
% 	\begin{equation}
% 		F = \overbar{\bigcap_{\ell \in \mathcal{L}} H^{\mathfrak{s}^\prime(i)}_\ell},
% 	\end{equation}
% 	where
% 	\begin{itemize}
% 		\item $\mathfrak{s}(\ell) = \mathfrak{s}^\prime(\ell)$ for all $\ell 
% 			\in \{\ell^\prime | \mathfrak{s}^\prime(\ell^\prime) \neq 0 \}$; and

% 		\item $\text{\rmfamily rank}(\{H^0_\ell | \ell \in \mathcal{L}, 
% 			\mathfrak{s}^\prime(\ell) = 0\} = m-m^\prime$.
% 	\end{itemize}
% 	Note that for an $m$-dimensional region $R$, $\overbar{R}$ is an 
% 	$m$-dimensional face of $R$. $0$-dimensional faces are called 
% 	\textbf{vertices}.
% \end{definition}

\begin{theorem}[\cite{StanleyIntroductionHyperplaneArrangements}]
\label{thm:arrangement_regions_bound}
	Let $\mathcal{H}$ be an arrangement of $N$ hyperplanes in dimension $n$. 
	Then $|\mathcal{R}_\mathcal{H}|$ is at most $\sum_{k=0}^n {N \choose k}$.
\end{theorem}

\begin{remark}
	Note that for a fixed dimension, $n$, the bound $\sum_{k=0}^n \negthinspace 
	{N \choose k}$ grows like $O(N^n\negthinspace/n!)$, i.e. sub-exponentially 
	in $N$.
\end{remark}

%% file: framework.tex
\section{Main Results} % (fold)
\label{sec:framework}

% We will start this section with a description of the NN verification problem to 
% provide the necessary context for our results. Then we will describe the three 
% main results in this paper:

% \begin{enumerate}[label={\itshape{(\roman*)}}]
% 	\item information about an NN that makes it polynomially verifiable, 
% 		namely that its linear regions are composed of regions from a known 
% 		hyperplane arrangement (details in Section 
% 		\ref{sec:sufficient_infor_for_verification});

% 	\item two types of NNs -- shallow NNs and TLL NNs -- for which 
% 		information sufficient to polynomially verify the NN is accessible in 
% 		polynomial time (details in Section 
% 		\ref{sec:polynomially_verifiable_nn_forms}); and

% 	% \item a polynomial procedure for converting an arbitrary NN to a 
% 	% 	polynomially verifiable form (a TLL NN) up to $\epsilon$ error (details 
% 	% 	in Section \ref{sec:tll_compiler}).
% \end{enumerate}

\subsection{NN Verification Problem} % (fold)
\label{sub:nn_verification_problem}

We take as a starting point the following verification problem for NNs.

\begin{problem}
\label{prob:generic_nn_verification}
	Let $\nn_\Theta: \mathbb{R}^n \rightarrow \mathbb{R}^m$ be a NN with at 
	least two layers. Furthermore, assume that there are two convex, bounded, 
	full-dimensional polytopes $P_x \subset \mathbb{R}^n$ and $P_y \subset 
	\mathbb{R}^m$ with H-representations\footnote{An H-representation of a 
	polytope is (possibly non-minimal) representation of that polytope as an 
	intersection of half-spaces.} given as follows:
	\begin{itemize}
		\item $ P_x \triangleq \cap_{i=1}^{\mathsf{N}_x} 
			\overbar{H^{-1}_{\ell_{x,i}}} \subset \mathbb{R}^n$ where 
			$\ell_{x,i} : \mathbb{R}^n \rightarrow \mathbb{R}$ is an affine map 
			for each $i = 1, \dots, \mathsf{N}_x$; and

		\item $P_y \triangleq \cap_{i=1}^{\mathsf{N}_y} 
			\overbar{H^{-1}_{\ell_{x,i}}} \subset \mathbb{R}^m $ where 
			$\ell_{y,i} : \mathbb{R}^m \rightarrow \mathbb{R}$ is an affine map 
			for each $i = 1, \dots, \mathsf{N}_y$.
	\end{itemize}
	Then the verification problem is to decide whether the following formula is 
	true:
	\begin{equation}
	\label{eq:verification_formula}
		\forall x \in P_x \subset \mathbb{R}^n . \big( \nn_{\Theta}\negthinspace(x) \in P_y \subset \mathbb{R}^m \big).
	\end{equation}
	If \eqref{eq:verification_formula} is true, the problem is \textbf{SAT}; 
	otherwise,it is \textbf{UNSAT}.
\end{problem}
We proceed with this formulation of Problem \ref{prob:generic_nn_verification} 
for simplicity, and to emphasize the verification complexity in terms of NN 
parameters. \textbf{Even so, our proposed algorithm evaluates NNs on regions 
where they are affine}, and on such regions, verifying the input/output 
property in \eqref{eq:verification_formula} is essentially the same as 
verifying a control-relevant property such as
\begin{equation}\label{eq:dynamical_constraint}
	\mathfrak{L}( x, \nn_\Theta(x)) \in P_y
\end{equation}
(for a linear function $\mathfrak{L}$). Examples of 
\eqref{eq:dynamical_constraint} appear in forward invariance verification of 
LTI systems (see Section \ref{sub:lti_system_forward_invariance_verification}) 
and verifying autonomous robots controlled by NNs 
\cite{SunFormalVerificationNeural2019}.

% subsection nn_verification_problem (end)

\subsection{Main Theorems} % (fold)
\label{sub:complexity_of_nn_verification}

The main results of this paper consist of showing that Problem 
\ref{prob:generic_nn_verification} can be solved in polynomial time complexity 
in the number of neurons for two classes of networks. In particular, we state 
the following two theorems.

\begin{theorem}\label{thm:main_shallow_nn_thm}
	Let $\Theta = ((n,\mathfrak{n}), (\mathfrak{n},m))$ be a shallow network 
	with $\mathfrak{n}$ neurons. Now consider an instance of Problem 
	\ref{prob:generic_nn_verification} for this network: i.e. fixed dimensions 
	$n$ and $m$, and fixed constraint sets $P_x$ and $P_y$. Then there is an 
	algorithm that solves this instance of Problem 
	\ref{prob:generic_nn_verification} in polynomial time complexity in 
	$\mathfrak{n}$. This algorithm has a worst case runtime of
	% \begin{multline}
	% 	% O(\mathsf{N}_y \cdot \mathfrak{n} \cdot (m \cdot \mathfrak{n} + \mathsf{N}_x)^{n+1}) \cdot \mathtt{Cplxty}(\mathtt{LP}(m \cdot \mathfrak{n} + \mathsf{N}_x,n))^2
	% 	O\Big((n\cdot \pmb{N}^{2n+2} \cdot \log \pmb{N} + \mathsf{N}_y\cdot \pmb{N}^{n} ) \cdot \mathtt{Cplxty}(\mathtt{LP}(\pmb{N},n)) + \\ \pmb{N}^n \cdot m \cdot \mathfrak{n} \cdot m \Big)
	% \end{multline}
	\begin{equation}
		% O(\mathsf{N}_y \cdot \mathfrak{n} \cdot (m \cdot \mathfrak{n} + \mathsf{N}_x)^{n+1}) \cdot \mathtt{Cplxty}(\mathtt{LP}(m \cdot \mathfrak{n} + \mathsf{N}_x,n))^2
		\mathsf{N}_y \cdot O( m \cdot n^2 \cdot \mathfrak{n}^{n+2} / n! ) \cdot \mathtt{Cplxty}(\mathtt{LP}(\mathfrak{n}+ \mathsf{N}_x,n))
	\end{equation}
	where $\mathtt{Cplxty}(\mathtt{LP}(N,n))$ is the complexity of solving a 
	linear program in dimension $n$ with $N$ constraints.  
\end{theorem}

\begin{theorem}\label{thm:main_tll_nn_thm}
	Let $\multitllparams{N}{M}{m}$ be a multi-output TLL network. Now consider 
	an instance of Problem \ref{prob:generic_nn_verification} for this network: 
	i.e. fixed dimensions $n$ and $m$, and fixed constraint sets $P_x$ and 
	$P_y$. Then there is an algorithm that solves this instance of Problem 
	\ref{prob:generic_nn_verification} in polynomial time complexity in $N$ and 
	$M$. This algorithm has a worst case runtime of
	% \begin{multline}
	% 	O\Big((n \cdot \pmb{N}^{2n+2} \cdot \log \pmb{N} + \mathsf{N}_y\cdot \pmb{N}^{n}) \cdot \mathtt{Cplxty}(\mathtt{LP}(\pmb{N},n)) \\ + \pmb{N}^n \cdot m \cdot (N^2 + M \cdot N \cdot \log N)\Big)
	% \end{multline}
	\begin{equation}
		\mathsf{N}_y \cdot O(m^{n+2} \cdot n \cdot M \cdot N^{2n+3}/n! ) \cdot \mathtt{Cplxty}(\mathtt{LP}(m \cdot N^2 + \mathsf{N}_x,n)) \notag
		% \\ + \pmb{N}^n \cdot m \cdot (N^2 + M \cdot N \cdot \log N)\Big)
	\end{equation}
	where $\mathtt{Cplxty}(\mathtt{LP}(N,n\negthinspace))$ is the complexity of 
	solving a linear program in dimension $n$ with $N$ constraints. The 
	algorithm is thus polynomial in the number of neurons in 
	$\multitllparams{N}{M}{m}$, since the number of neurons depends 
	polynomially on $N$ and $M$.
\end{theorem}

\noindent In particular, Theorem \ref{thm:main_shallow_nn_thm} and Theorem 
\ref{thm:main_tll_nn_thm} explicitly indicate that the difficulty in verifying 
their respective classes of NNs grows only polynomially in the complexity of 
the network, all other parameters of Problem \ref{prob:generic_nn_verification} 
held fixed. Note also that the polynomial complexity of these algorithms 
depends on the existence of polynomial-time solvers for linear programs, but 
such solvers are well known to exist (see e.g. 
\cite{NemirovskiInteriorpointMethodsOptimization2008}).

It is important to note that Theorem \ref{thm:main_shallow_nn_thm} and Theorem 
\ref{thm:main_tll_nn_thm} both identify algorithms that are exponential in the 
input dimension to the network. Thus, these results do not contradict the 3-SAT 
embedding of \cite{KatzReluplexEfficientSMT2017a}. Indeed, given that a TLL NN 
can represent any CPWA function \cite{FerlezAReNAssuredReLU2020} -- including 
the 3-SAT gadgets used in \cite{MontufarNumberLinearRegions2014} -- it follows 
directly that the satisfiability of any 3-SAT formula can be encoded as an 
instance of Problem \ref{prob:generic_nn_verification} for a TLL NN. Since the 
input dimensions of both NNs are the same, the conclusion of 
\cite{KatzReluplexEfficientSMT2017a} is preserved, even though the TLL may use 
a different number of neurons.

Finally, it is essential to note that the results in Theorem 
\ref{thm:main_shallow_nn_thm} and Theorem \ref{thm:main_tll_nn_thm} connect the 
difficulty of verifying a TLL NN (resp. shallow NN) to the \emph{size} of the 
network not the \emph{expressivity} of the network. The semantics of the TLL NN 
in particular make this point especially salient, since each distinct affine 
function represented in the output of a TLL NN can be mapped directly to 
parameters of the TLL NN itself (see Proposition \ref{prop:linear_regions_tll} 
in Section \ref{sec:polynomially_verifiable_nn_forms}). In particular, consider 
the deep NNs exhibited in \cite[Corrollary 6]{MontufarNumberLinearRegions2014}: 
this parameterized collection of networks expresses a number of unique affine 
functions that grows exponentially in the number of neurons in the network 
(i.e. as a function of the number of layers in the network). Consequently, the 
size of a TLL required to implement one such  network would likewise grow 
exponentially in the number of neurons deployed in the original network. Thus, 
although a TLL NN may superficially seem ``easy'' to verify because of Theorem 
\ref{thm:main_tll_nn_thm}, the efficiency in verifying a TLL NN form could mask 
the fact that a particular TLL NN implementation is less parameter efficient 
than some other representation (in terms of neurons required). Ultimately, this 
trade-off will not necessarily be universal, though, since TLL NNs also have 
mechanisms for parametric \emph{efficiency}: for example, a particular local 
linear function need only be implemented once in a TLL NN, no matter how many 
disjoint regions on which it is activated (as in the case of implementing 
interpolated zero-order-hold functions, such as in 
\cite{FerlezTwoLevelLatticeNeural2020}).

\subsection{Proof Sketch of Main Theorems} % (fold)
\label{sub:proof_sketch_of_main_result}

\subsubsection{Core Theorem: Polynomial-time Enumeration of Hyperplane Regions} % (fold)
\label{ssub:core_lemma_polynomial_enumeration_of_hyperplane_regions}

The algorithms that witness the claims in Theorem \ref{thm:main_shallow_nn_thm} 
and \ref{thm:main_tll_nn_thm} have the same broad structure:
\begin{enumerate}[left=\parindent,label={\itshape Step \arabic*:}]
	\item For the architecture in question, choose (in polynomial time) a 
		hyperplane arrangement with the following three properties:
		\begin{enumerate}[label={\itshape (\alph*)}]
			\item The number of hyperplanes is polynomial in the number of 
				network neurons;

			\item $P_x$ is contained in the union of the closure of regions 
				from this arrangement; and

			\item Problem \ref{prob:generic_nn_verification} can be solved 
				in polynomial time on the closure of any region, $R$, in this 
				arrangement -- i.e. Problem \ref{prob:generic_nn_verification} 
				with $P_x$ replaced by $\overbar{R}$ can be solved in 
				polynomial time.
		\end{enumerate}

	\item Iterate over all of the regions in this arrangement, and for each 
		region, solve Problem \ref{prob:generic_nn_verification} with $P_x$ 
		replaced by $\overbar{R}\cap P_x$.
\end{enumerate}
The details of \emph{Step 1} vary depending on the architecture of the network 
being verified (Theorem \ref{thm:main_shallow_nn_thm} and Theorem 
\ref{thm:main_tll_nn_thm}). However, no matter the details of \emph{Step 1}, 
this proof structure depends on a polynomial time algorithm to traverse the 
regions in a hyperplane arrangement. But it is known that there exist 
polynomial algorithms to perform such enumerations. The following result from 
\cite{AvisReverseSearchEnumeration1996} is one example; an ``optimal'', if not 
practical, option is \cite{EdelsbrunnerConstructingArrangementsLines1986}.

\begin{theorem}[\cite{AvisReverseSearchEnumeration1996} Theorem 3.3]
	\label{lem:traverse_regions_lemma}
	Let $\mathcal{L} = \{\ell_1, \dots, \ell_N\}$ be a set of affine functions, 
	$\ell_i : \mathbb{R}^n \rightarrow \mathbb{R}$, that can be accessed in 
	$O(1)$ time, and let $\mathcal{HL} = \{H^0_\ell | \ell \in \mathcal{L}\}$ 
	be the associated hyperplane arrangement.

	Then there is an algorithm to traverse all of the regions in $\mathcal{HL}$ 
	that has runtime
	\begin{equation}
		O(n\cdot N^{n+1}/{n!}) \cdot \mathtt{Cplxty}(\mathtt{LP}(N,n))
	\end{equation}
	where $\mathtt{Cplxty}(\mathtt{LP}(N,n))$ is the complexity of solving a 
	linear program in dimension $n$ with $N$ constraints.
\end{theorem}

\noindent Note that there is more to Theorem \ref{lem:traverse_regions_lemma} 
than just the sub-exponential bound on the number of regions in a hyperplane 
arrangement (see Theorem \ref{thm:arrangement_regions_bound} in Section 
\ref{sec:preliminaries}). Indeed, although there are only $O(N^n/n!)$ regions 
in an arrangement of $N$ hyperplanes in dimension $n$, it must be inferred 
which of the $2^N$ possible activations correspond to \emph{valid} regions. 
That this is possible in polynomial time is the main contribution of Theorem 
\ref{lem:traverse_regions_lemma}, and thus facilitates the results in this 
paper.
%hence the main facilitator of the other 
%results in this paper.

% subsubsection core_lemma_polynomial_enumeration_of_hyperplane_regions (end)

\subsubsection{Theorem \ref{thm:main_shallow_nn_thm} and Theorem \ref{thm:main_tll_nn_thm}} % (fold)
\label{ssub:theorem_thm:main_shallow_nn_thm_and_theorem_thm:main_tll_nn_thm}

Given Theorem \ref{lem:traverse_regions_lemma}, the proofs of Theorem 
\ref{thm:main_shallow_nn_thm} and Theorem \ref{thm:main_tll_nn_thm} depend on 
finding a suitable hyperplane arrangement, as described in \emph{Step 1}.% in 
%the previous subsection.

In both cases, we note that the easiest closed convex polytope on which to 
solve Problem \ref{prob:generic_nn_verification} is one on which the underlying 
NN is affine. Indeed, suppose for the moment that $\nn\subarg{\Theta}$ is 
affine on the entire constraint set $P_x$ with $\nn\subarg{\Theta} = \ell_0$ on 
this domain. Under this assumption, solving the verification problem for a 
single output constraint, $\ell_{y,i}$, entails solving the following linear 
program:
\begin{align}
	y_i &= \max (\ell_{y,i}\circ\ell_0)(x) \notag \\
		&~ \text{ s.t. }\ell_{x,i^\prime}(x) \leq 0 \text{ for } i^\prime = 1, \dots, \mathsf{N}_x.
\end{align}
Of course if $y_i > 0$, then Problem \ref{prob:generic_nn_verification} is 
UNSAT under these assumptions; otherwise it is SAT for the constraint 
$\ell_{y,i}$ and the next output constraint needs to be considered. Given the 
known (polynomial) efficiency of solving linear programs, it thus makes sense 
to select a hyperplane arrangement for \emph{Step 1} with the property that the 
NN is affine on each region of the arrangement. Although this is a difficult 
problem for a general NN, the particular structure of shallow NNs and TLL NNs 
allow such a selection to be accomplished efficiently.

% the following sequence of observations. First, that any ReLU NN -- no matter 
% the architecture -- partitions its input space into closed convex regions on 
% which the NN is a purely affine function of its inputs. Second, it is thus 
% possible verify Problem \ref{prob:generic_nn_verification} \emph{on such a 
% region} by means of a single linear program per hyperplane comprising $P_y$, 
% provided that region has a known representation and the affine function 
% implemented by the NN is likewise known. That is if an UNSAT input is found on 
% one such region, then the whole problem is UNSAT; on the other hand, if all 
% such regions lack an UNSAT input, then the whole problem is SAT.

% This suggests an obvious strategy for proving both claims: 

% \begin{enumerate}[label={\itshape(\arabic*)}]
% 	\item Enumerate each region on the input space for which the NN to be 
% 		verified is purely affine.

% 	\item Determine for each such region the affine function implemented by 
% 		the NN of interest.

% 	\item Run a single linear program to determine SAT/UNSAT on that region 
% 		with respect to each hyperplane comprising the output constraint set 
% 		$P_y$.
% \end{enumerate}

To this end, we make the following definition.

\begin{definition}[Switching Affine Function/Hyperplane]
	Let $\nn\subarg{\Theta} : \mathbb{R}^n \rightarrow \mathbb{R}^m$ be a NN. A 
	set of affine functions $\mathcal{S} = \{\ell_1, \dots, \ell_N\}$ with 
	$\ell_\iota : \mathbb{R}^n \rightarrow \mathbb{R}$ is said to be a set of 
	\textbf{switching affine functions for} $\nn\subarg{\Theta}$ if 
	$\nn\subarg{\Theta}$ is affine on every region of the hyperplane 
	arrangement $\mathcal{HS} = \{H^0_{\ell} | \ell \in \mathcal{S} \}$. 
	$\mathcal{HS}$ is then said to the be an \textbf{arrangement of switching 
	hyperplanes} of $\nn\subarg{\Theta}$.
\end{definition}

\noindent For both shallow NNs and TLL NNs, we will show that a set of 
switching hyperplanes is immediately evident (i.e. in polynomial complexity) 
from the parameters of those architectures directly; this satisfies \emph{Step 
1(b)}. However it also further implies that this choice of switching 
hyperplanes has a \emph{number} of hyperplanes that is polynomial in the number 
of neurons in either network; this satisfies \emph{Step 1(a)}.

%% file: verifier.tex
\section{Polynomial-time Algorithm to Verify Shallow NNs} % (fold)
\label{sec:proof_of_theorem_thm:main_shallow_nn_thm}
This section consists of several propositions that address the various aspects 
of \emph{Step 1}, as described in Subsection 
\ref{ssub:core_lemma_polynomial_enumeration_of_hyperplane_regions}. Theorem 
\ref{thm:main_shallow_nn_thm} is a direct consequence of these propositions.
%The section 
% will conclude with a formal statement of the proof of Theorem 
% \ref{thm:main_shallow_nn_thm}.

\begin{proposition}\label{prop:linear_regions_shallow}
	Let $\Theta = ((W^{|1},b^{|1}), (W^{|2},b^{|2}))$ be a shallow NN with 
	$\text{Arch}(\Theta) = ((n, \mathfrak{n}),(\mathfrak{n},m))$. Then the set 
	of affine functions
	\begin{equation}
		\mathcal{S}(\Theta) \triangleq \{ \mathscr{L}^i\subarg{W^{|1},b^{|1}}, i = 1, \dots, \mathfrak{n}  \}
	\end{equation}
	is a set of switching affine functions for $\Theta$, and 
	$\mathcal{HS}(\Theta) = \{H^0_\ell | \ell \in \mathcal{S}(\Theta)\}$ is a 
	set of switching hyperplanes.
\end{proposition}

\begin{proof}
	A region in the arrangement $\mathcal{HS}(\Theta)$ exactly assigns to each 
	neuron a status of strictly active or strictly inactive.
	% -- i.e. the output 
	%of the neuron is strictly greater than zero or strictly less than zero.
	But forcing a particular activation on \emph{each} of the neurons forces 
	the shallow NN to operate on an affine region.
\end{proof}

\begin{proposition}\label{prop:shallow_linfn_access}
	Let $\Theta = ((W^{|1},b^{|1}), (W^{|2},b^{|2}))$ be a shallow NN with 
	$\text{Arch}(\Theta) = ((n, \mathfrak{n}),(\mathfrak{n},m))$, and let 
	$\mathcal{HS}(\Theta)$ be as in Proposition 
	\ref{prop:linear_regions_shallow}. Then the complexity of determining the 
	active linear function on a region of $\mathcal{HS}(\Theta)$ is at most
	\begin{equation}
		O(m\cdot \mathfrak{n} \cdot n).
	\end{equation}
\end{proposition}
\begin{proof}
	This runtime is clearly dominated by the cost of doing the matrix 
	multiplication $W^{|2} \cdot W^{|1}$. Given that $\text{Arch}(\Theta) = 
	((n, \mathfrak{n}),(\mathfrak{n},m))$, this operation has the claimed 
	runtime.
\end{proof}

% \begin{proposition}
% 	The arrangement $\mathcal{H} \triangleq \mathcal{HS}(\Theta) \cup 
% 	\{H^0_{\ell_{x,i}} | i = 1, \dots, \mathsf{N}_x \}$ has the property that 
% 	$P_x$ is the union of the closure of regions from $\mathcal{H}$, and each 
% 	region in $\mathcal{H}$ is strictly contained in a region from 
% 	$\mathcal{HS}(\Theta)$. Hence, $\mathcal{H}$ is also a switching hyperplane 
% 	arrangement and Propositions \ref{prop:linear_regions_shallow} and 
% 	\ref{prop:shallow_linfn_access} apply just as well to regions of 
% 	$\mathcal{H}$ with run times the same as for $\mathcal{HS}(\Theta)$. 
% \end{proposition}

% \begin{proof}
% 	This follows trivially by definition of a region and the fact that we are 
% 	merely adding hyperplanes to the arrangement $\mathcal{HS}(\Theta)$.
% \end{proof}

\noindent Theorem \ref{thm:main_shallow_nn_thm} now follows.

\begin{proof}
	(Theorem \ref{thm:main_shallow_nn_thm}.) First note that 
	$|\mathcal{HS}(\Theta)| = \mathfrak{n}$. Thus, by Proposition 
	\ref{prop:linear_regions_shallow}, the closure of regions from 
	$\mathcal{HS}(\Theta)$ covers $\mathbb{R}^n$, and $\Theta$ is an affine 
	function on each such closure.

	Thus, an algorithm to solve Problem \ref{prob:generic_nn_verification} for 
	$\Theta$ can be obtained by enumerating the regions in 
	$\mathcal{HS}(\Theta)$ using the algorithm from Theorem 
	\ref{lem:traverse_regions_lemma}, and for each such region, $\mathfrak{s}$ 
	(see Definition \ref{def:hyperplane_region}), solving one linear program:
	\begin{align}
		y_i &= \max (\ell_{y,i}\circ\ell_0)(x) \notag \\
		&~ \text{ s.t. }\ell_{x,i^\prime}(x) \leq 0 \text{ for } i^\prime = 1, \dots, \mathsf{N}_x \notag \\
		&~ \text{ and }\mathfrak{s}(\ell)\cdot\ell(x) \leq 0 \text{ for } \ell \in \mathcal{S}(\Theta).
	\end{align}
	for each output polytope constraint $i \in 1, \dots, \mathsf{N}_y$. The 
	claimed runtime then follows directly by incorporating the cost of 
	computing the active affine function on each such region (Proposition 
	\ref{prop:shallow_linfn_access}) and bounding the size of each enumeration 
	LP (Theorem \ref{lem:traverse_regions_lemma}) by the size of the  LP above, 
	which has at most $\mathfrak{n} \negthinspace + \negthinspace \mathsf{N}_x$ 
	constraints in dimension $n$.
\end{proof}

% \noindent Now we can finally state the proof of Theorem \ref{thm:main_shallow_nn_thm}.

% \begin{proof}
% 	(Theorem \ref{thm:main_shallow_nn_thm}) We need to traverse the hyperplane 
% 	arrangement $\mathcal{HS}(\theta)$ $\cup$ $\{ H^0_{\ell_{x,i}} | i = 1, 
% 	\dots, \mathsf{N}_x \}$, which has $\mathfrak{n} + \mathsf{N}_x$ 
% 	hyperplanes. By Theorem \ref{thm:arrangement_regions_bound}, this 
% 	arrangement has $O((\mathfrak{n} + \mathsf{N}_x)^n)$ regions, so by Lemma 
% 	\ref{lem:traverse_regions_lemma}, we need at most $O(n \cdot (\mathfrak{n} 
% 	+ \mathsf{N}_x)^{2n+2} \cdot \log (\mathfrak{n} + \mathsf{N}_x) )$ calls to 
% 	an LP solver to traverse all of these regions
% 	%-- i.e. we need $\mathfrak{n} 
% 	% + \mathsf{N}_x$ LP calls per region to find the minimal H-representation of 
% 	% that region using Algorithm \ref{alg:find_successors}, in addition to the 
% 	% overhead associated with Algorithm \ref{alg:traverse_regions}. 
% 	Then, on 
% 	each region, we need to access the active linear function, so we add the 
% 	run time described in Proposition \ref{prop:shallow_linfn_access} times the 
% 	number of regions, i.e. $O( (\mathfrak{n} + \mathsf{N}_x)^n \cdot m \cdot 
% 	\mathfrak{n} \cdot n )$. Finally, we need to run one linear program per 
% 	region per output constraint. This comes at a total cost of $\mathsf{N}_y 
% 	\cdot (\mathfrak{n} + \mathsf{N}_x)^n$ calls to the LP solver. This 
% 	explains the runtime expression claimed in the Theorem.
% \end{proof}

% section proof_of_theorem_thm:main_shallow_nn_thm (end)

\section{Polynomial-time Algorithm to Verify TLL NN} % (fold)
\label{sec:polynomially_verifiable_nn_forms}
This section consists of several propositions that address the various aspects 
of \emph{Step 1}, as described in Subsection 
\ref{ssub:core_lemma_polynomial_enumeration_of_hyperplane_regions}. Theorem 
\ref{thm:main_tll_nn_thm} is a direct consequence of these propositions.
%  The section 
% concludes with a formal statement of the proof of Theorem 
% \ref{thm:main_tll_nn_thm}.
% \begin{definition}[Switching Functions/Hyperplanes of a TLL NN]
% 	Let $\tllparams{N}{M}$ be a scalar TLL NN as described in Definition 
% 	\ref{def:multi-output_tll}, and let $W_\ell$ and $b_\ell$ be the parameters 
% 	associated with its first layer. Then the set of \textbf{switching 
% 	functions of } $\Xi_{N,M}$ is given by
% 	\begin{equation}
% 		\mathcal{S} (\tllparams{N}{M}) \triangleq \{ \mathscr{L}^i \subarg{ W_\ell,b_\ell } - \mathscr{L}^j \subarg{ W_\ell,b_\ell } | i,j\in \{1, \dots, N\} \wedge i \leq j \},
% 	\end{equation}
% 	and the arrangement of \textbf{switching hyperplanes of} $\Xi_{N,M}$ is 
% 	given by
% 	\begin{equation}
% 		\mathcal{HS} (\Xi_{N,M}) \triangleq \left\{ H^{0}_{\ell} \;\Big|\; \ell \in \mathcal{S}(\tllparams{N}{M}) \right\}.
% 	\end{equation}

% 	For a multi-output TLL NN, $\Xi^{(m)}_{N,M}$, the set of switching 
% 	functions, $\mathcal{S} (\multitllparams{N}{M}{m}) \triangleq 
% 	\cup_{k=1}^{m} \mathcal{S} (\tllparams{N}{M}^k )$ where 
% 	$\tllparams{N}{M}^k$ is the $k^\text{th}$ component TLL of 
% 	$\multitllparams{N}{M}{m}$. $\mathcal{HS} (\multitllparams{N}{M}{m})$ is 
% 	defined as above from $\mathcal{S} (\multitllparams{N}{M}{m})$.
% \end{definition}
\begin{proposition}\label{prop:linear_regions_tll}
	Let $\multitllparams{N}{M}{m}$ be a TLL NN. Then define
	\begin{equation}
		\mathcal{S}\negthinspace(\tllparams{N}{M}^\kappa\negthinspace) \negthinspace \triangleq \negthinspace \{ \negthinspace\mathscr{L}^i\subarg{ W_\ell^\kappa\negthinspace,\negthinspace b_\ell^\kappa } \negthickspace - \negthickspace \mathscr{L}^j \subarg{ W_\ell^\kappa\negthinspace,\negthinspace b_\ell^\kappa } |
		i < j \negthinspace \in \negthinspace \{\negthinspace1, .., \negthinspace N\negthinspace\} \negthinspace \}
		% \notag
	\end{equation}
	and $\mathcal{H\negthinspace S}(\tllparams{N}{M}^\kappa) \negthickspace 
	\triangleq \negthickspace \{H^0_\ell | \ell \in 
	\mathcal{S}(\tllparams{N}{M}^\kappa) \}$. Furthermore, define 
	$\mathcal{S}(\multitllparams{N}{M}{m}) \negthickspace \triangleq 
	\negthickspace \cup_{\kappa = 1}^m \mathcal{S}(\tllparams{N}{M}^\kappa)$ 
	and $\mathcal{H\negthinspace S}(\multitllparams{N}{M}{m}) \negthickspace 
	\triangleq \negthickspace \cup_{\kappa =1}^{m} \mathcal{H\negthinspace 
	S}(\tllparams{N}{M}^\kappa).$

	Then $\mathcal{S}(\multitllparams{N}{M}{m})$ is a set of switching affine 
	functions  for $\multitllparams{N}{M}{m}$, and the $\kappa^\text{th}$ 
	component of $\multitllparams{N}{M}{m}$ is an affine function on each 
	region of $\mathcal{HS}(\multitllparams{N}{M}{m})$ and is exactly equal to 
	$\mathscr{L}^i \subarg{ W^\kappa_\ell, b^\kappa_\ell }$ for some $i$.
\end{proposition}
\begin{proof}
	Let $R$ be a region in $\mathcal{HS} (\multitllparams{N}{M}{m})$. It is 
	obvious that such a region is contained in exactly one region from each of 
	the component-wise arrangements $\mathcal{HS} (\tllparams{N}{M}^\kappa)$, 
	so it suffices to show that each component TLL is linear on the regions of 
	its corresponding arrangement.

	Thus, let $R_\kappa$ be a region in $\mathcal{HS} (\tllparams{N}{M}^\kappa 
	)$. We claim that $\tllparams{N}{M}^\kappa$ is linear on $R_\kappa$. To see 
	this, note by definition of a region, there is an indexing function 
	$\mathfrak{s} : \mathcal{S} (\tllparams{N}{M}^\kappa) \rightarrow 
	\{-1,+1\}$ such that $R_k = \bigcap_{\ell \in \mathcal{S} 
	(\tllparams{N}{M}^\kappa )} H^{\mathfrak{s}(\ell)}_\ell$. Thus, $R_k$ is a 
	unique order region by construction: each such half-space identically 
	orders the outputs of two linear functions, and since $R_k$ is 
	$n$-dimensional it is contained in just such a half space for each and 
	every possible pair.
	%That is
	% \begin{equation}
	% 	\forall x \in R_k \;.\; \llbracket W^\kappa_\ell x + b^\kappa_\ell \rrbracket_{\hat{\iota}_1} < 
	% 	\dots < \llbracket W^\kappa_\ell x + b^\kappa_\ell \rrbracket_{\hat{\iota}_N}
	% \end{equation}
	% for some sequence $\hat{\iota}_\kappa \subset \{1, \dots, N\}$.
	% Thus, for each $j = 1 , \dots, M$, there exists an index $\iota_j \in \{1, 
	% \dots, N\}$ such that
	% \begin{equation}
	% 	\forall x \in R_k . \;\left( \nn\subarg{\Theta_{\min_N}}( S_j (W^\kappa_\ell x + b^\kappa_\ell))
	% 	=
	% 	\llbracket W^\kappa_\ell x + b^\kappa_\ell \rrbracket_{\iota_j}.
	% 	\right)
	% \end{equation}
	% That is each of the min terms in $\tllparams{N}{M}^\kappa$ is exactly equal 
	% to one of the linear functions $\llbracket W^\kappa_\ell x + b^\kappa_\ell 
	% \rrbracket_\iota$ on $R_\kappa$. Therefore,
	% \begin{equation}
	% 	\forall x \in R_\kappa \;.\; \nn\subarg{\tllparams{N}{M}^\kappa}(x) =
	% 	\nn\subarg{\Theta_{\max_M}}\negthinspace(
	% 		\left[
	% 		\begin{smallmatrix}
	% 			\llbracket W^\kappa_\ell x + b^\kappa_\ell \rrbracket_{\iota_1} \\
	% 			\vdots \\
	% 			\llbracket W^\kappa_\ell x + b^\kappa_\ell \rrbracket_{\iota_M}
	% 		\end{smallmatrix}
	% 		\right]
	% 	).
	% \end{equation}
	Applying the unique ordering property of $R_\kappa$ to the definition of 
	the TLL NN implies that there exists an index $\iota \in \{1, \dots, N\}$ 
	such that $\nn \subarg{ \tllparams{N}{M}^\kappa } (x) = \llbracket 
	W^\kappa_\ell x + b^\kappa_\ell \rrbracket_{\iota}$ for all $x \in 
	R_\kappa$.
	% Finally, the result on the unique order regions follows from the definition 
	% of a unique order region in terms of a hyperplane arrangement, together 
	% with well-known bounds on the number of regions in such arrangements. See 
	% Definition \ref{def:order_region} and \cite{FerlezAReNAssuredReLU2020}; the 
	% relevant bounds are a consequence of Zaslavsky's Theorem 
	% \cite{StanleyIntroductionHyperplaneArrangements}. Indeed, the conclusion 
	% follows directly by noting that the unique order regions from one component 
	% merely subdivide the unique order regions from another, by means of 
	% combining the associated hyperplanes from both components into one larger 
	% arrangement -- see Definition \ref{def:order_region}.
\end{proof}

\begin{proposition}\label{prop:tll_linfn_access}
	Let $\multitllparams{N}{M}{m}$ be a multi-output TLL NN,
	%as described in 
	%Definition \ref{def:multi-output_tll}
	and let $\mathcal{HS}\negthinspace(\multitllparams{N}{M}{m})$ be as in 
	Proposition \ref{prop:linear_regions_tll}. Then for any region $R$ of 
	$\mathcal{HS}\negthinspace(\multitllparams{N}{M}{m})$ the affine function 
	of $\multitllparams{N}{M}{m}$ that is active on $R$ can be found by a 
	polynomial algorithm of runtime
	\begin{equation}
		O(m \cdot M \cdot (N+1)).
	\end{equation}
\end{proposition}

\begin{proof}
	From the proof of Proposition \ref{prop:linear_regions_tll} we know that a 
	region in $\mathcal{HS}(\multitllparams{N}{M}{m})$ is a unique order region 
	for the linear functions $(W^\kappa_\ell,b^\kappa_\ell)$ of each component. 
	In other words, the indexing function of the region $R$ specifies a strict 
	ordering of each pair of linear functions from each component of 
	$\multitllparams{N}{M}{m}$. Thus, the region $R$, using its indexing 
	function $\mathfrak{s}_R$, \emph{pairwise-orders} the component-wise linear 
	functions on $R$.

	These pairwise comparisons can be used to identify the active affine 
	function on each $\min$ group, $\Theta_{\min_N}$ (of each output component) 
	by means of successive comparison in a bubble-sort-type way. Thus, 
	resolving the active function on each $\min$ group requires $N$ 
	comparisons, each of which is a direct look up in the region indexing 
	function, and hence $O(1)$. Moreover, the same argument applies to the 
	$\max$ operation for each component, only resolving the active affine 
	function there requires $M$ comparisons instead. Since there are $m \cdot 
	M$ $\min$ groups in total, and there are $m$ $\max$ groups in total, 
	resolving the active affine function runs in $O(m\cdot M \cdot N + m \cdot 
	M) = O(m \cdot M \cdot (N+1))$ as claimed.
	% can be turned into a sorted list of linear 
	% functions for each component at the cost of $O(m \cdot N^2)$, which 
	% explains the first term in the claimed runtime. Thus, it remains only to 
	% check each of the selection sets against these component-wise sorted lists 
	% of linear functions to see which affine function is active for each $\min$ 
	% term; this is then followed by checking which of the active functions from 
	% each $\min$ term survives as the active function after the $\max$ 
	% operation. A procedure for doing this involves at most $O(N \cdot \log N)$ 
	% operations per component per $\min$ term (using a binary search of the 
	% sorted list), followed by another $O(N \cdot \log N)$ comparisons per 
	% component to see which of the (unique) linear functions expressed by $\min$ 
	% terms survives the final max term.
\end{proof}

% \begin{proposition}
% 	The arrangement $\mathcal{H} \triangleq 
% 	\mathcal{HS}(\multitllparams{N}{M}{m}) \cup \{H^0_{\ell_{x,i}} | i = 1, 
% 	\dots, \mathsf{N}_x \}$ has the property that $P_x$ is the union of the 
% 	closure of regions from $\mathcal{H}$, and each region in $\mathcal{H}$ is 
% 	strictly contained in a region from 
% 	$\mathcal{HS}(\multitllparams{N}{M}{m})$. Hence, $\mathcal{H}$ is also a 
% 	switching hyperplane arrangement and Propositions 
% 	\ref{prop:linear_regions_tll} and \ref{prop:tll_linfn_access} apply just as 
% 	well to regions of $\mathcal{H}$ with run times the same as for 
% 	$\mathcal{HS}(\multitllparams{N}{M}{m})$. 
% \end{proposition}

% \begin{proof}
% 	This follows trivially by definition of a region and the fact that we are 
% 	merely adding hyperplanes to the arrangement 
% 	$\mathcal{HS}(\multitllparams{N}{M}{m})$.
% \end{proof}

\noindent Theorem \ref{thm:main_tll_nn_thm} now follows from these propositions.

\begin{proof}
	(Theorem \ref{thm:main_tll_nn_thm}.) This proof follows exactly the same 
	structure as the proof of Theorem \ref{thm:main_shallow_nn_thm}. The 
	salient differences are that $|\mathcal{S}(\multitllparams{N}{M}{m})| = m 
	\cdot N\cdot(N-1)/2 = O(m \cdot N^2)$ (Proposition 
	\ref{prop:linear_regions_tll}), and the cost of obtaining the active affine 
	function on each region thereof is now specified as $O(m\cdot M \cdot 
	(N+1))$ (Proposition \ref{prop:tll_linfn_access}). The claimed runtime for 
	TLL networks follows mutatis mutandis from Theorem 
	\ref{lem:traverse_regions_lemma}.
\end{proof}

%% file: experiments.tex
\section{Numerical Results} % (fold)
\label{sec:numerical_results}

\begin{figure*}%[htb!]
    \centering
	\begin{subfigure}[t]{0.49\textwidth}
		\hspace*{15pt}%f\vspace{-30pt}
		\raisebox{40pt}{\emph{(a)}~}\hspace{5pt}\includegraphics[width=0.8\linewidth,trim={0px 0cm 0px 0px},clip]{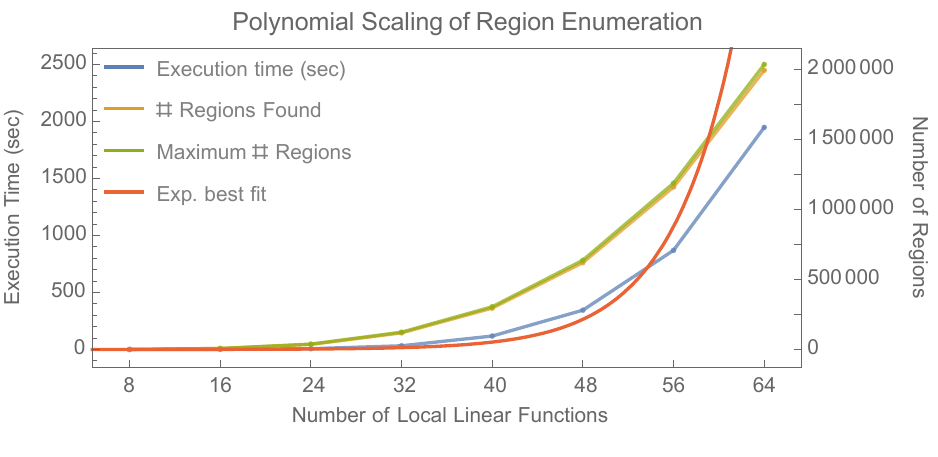}
		\vspace{-10pt}
	% \caption{Experiment {\bf 2}, ShieldNN OFF}
	\label{fig:2-off}
	% \vspace{-20pt}
	\end{subfigure}%\hspace*{-5pt}
	\begin{subfigure}[t]{0.49\textwidth}
		\hspace*{10pt}
		\raisebox{40pt}{\emph{(b)}~}\hspace{5pt}\includegraphics[width=0.7\linewidth,clip]{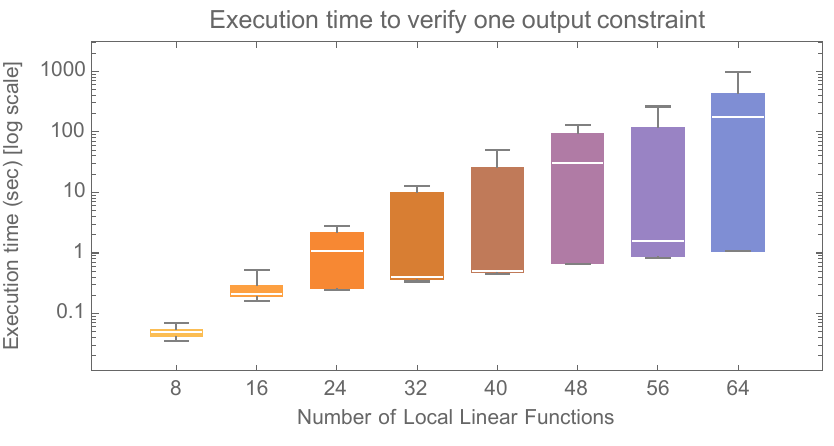}
		% \caption{Experiment {\bf 2}, ShieldNN ON
		% }
		\label{fig:2-on}
	\end{subfigure}

	\begin{subfigure}[t]{0.49\textwidth}
		\hspace*{14pt}%\vspace{30pt}
		\raisebox{40pt}{\emph{(c)}~}\hspace{5pt}\includegraphics[width=.7\linewidth,trim={0px 0cm 0px 0.08cm},clip]{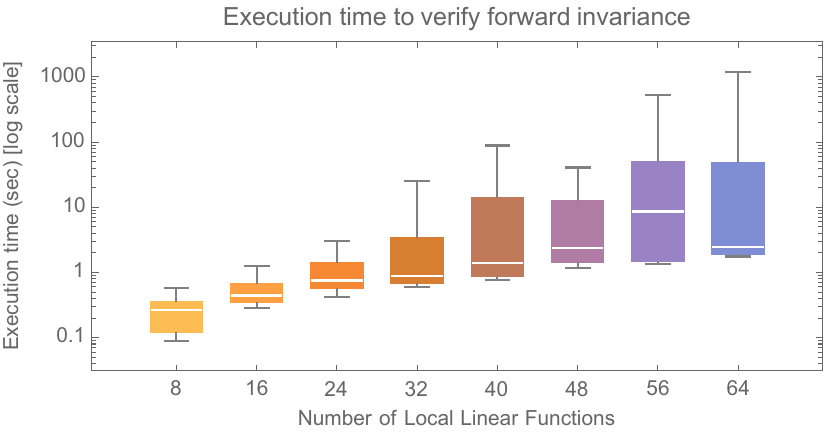}
	% \caption{Experiment {\bf 3A}, ShieldNN OFF}
	\label{fig:3a-off}
	% \vspace{-20pt}
	\end{subfigure}\hspace*{-12pt}
	\begin{subfigure}[t]{0.49\textwidth}
		%\hspace*{-10pt}
		\hspace{25pt}\vspace{2pt}
		\raisebox{40pt}{\emph{(d)}~}\hspace{5pt}\includegraphics[width=.63\linewidth,trim={0px 0cm 0px 0.06cm},clip]{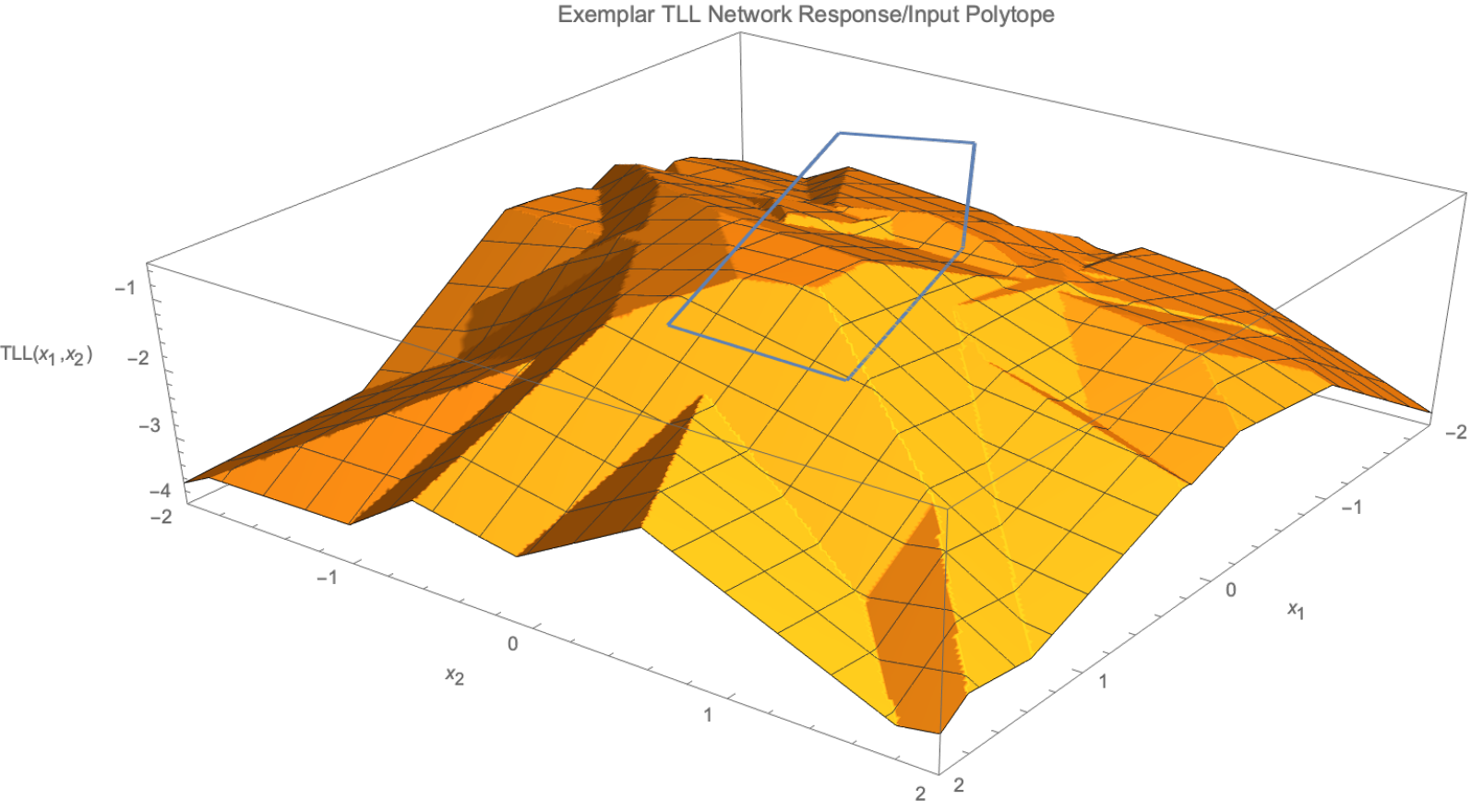}
		% \vspace{30pt}
		% \caption{Experiment {\bf 3A}, ShieldNN ON
		% }
		\label{fig:3a-on}
	\end{subfigure}
    %
%    \begin{subfigure}[t]{0.49\textwidth}
%		\includegraphics[width=1.\textwidth,trim={60px 0cm 60px 30px},clip]{figs/histograms/3_ON.png}
%		\caption{Experiment {\bf 3B}}
%		\label{fig:hist_3_3}
%	\end{subfigure}
\vspace{-4pt}
	\caption{\emph{(a)} Polynomial growth in number of regions and time needed to enumerate them (red curve is an exponential best-fit to execution time for reference); \emph{(b)} Box plot of execution times to verify a single random output constraint for each of several TLL sizes; \emph{(c)} Box plot of execution time to verify forward invariance of a polytopic set for each of several size TLLs; \emph{(d)} a TLL network and input constraint polytope used in output-constraint/forward-invariance experiments ($N=64$).}
	\label{fig:main_fig}
\vspace{-18pt}
\end{figure*}

To validate the claims we have made about the polynomial efficiency of the TLL 
verification problem, we implemented a version of the algorithm described in 
Section \ref{sec:polynomially_verifiable_nn_forms}. Then we conducted three 
separate experiments on a selection of randomly generated TLL networks of 
various sizes.

\begin{enumerate}[label={\itshape{\Alph*.}}]
	\item We used our tool to merely enumerate the regions in 
		$\mathcal{H\negthinspace S}(\tllparams{N}{M}^{(1)})$ (see Proposition 
		\ref{prop:linear_regions_tll}); this verifies that our implemented 
		hyperplane-region enumeration algorithm is in fact polynomial, and 
		confirms Theorem \ref{thm:arrangement_regions_bound}. 

	\item For each TLL network, we randomly generated a polytope in its 
		domain to serve as an input constraint for an instance of Problem 
		\ref{prob:generic_nn_verification}. We verified each such network/input 
		constraint with respect to a single, randomly generated output 
		constraint.

	\item Finally, we randomly generated an LTI system of the appropriate 
		dimension, and used our tool to check whether the same polytope 
		associated with each TLL network was forward invariant when said 
		network was used as a state feedback controller.
\end{enumerate}
These experiments were conducted on randomly generated TLL networks with $n=2$ 
and $m=1$ for sizes $N = 8, 16, 24, 32, 40, 48, 56, \text{ and } 64$, with 
$M=N$ for each network. We generated 20 instances of each size. A 3D plot of 
one such network is depicted in Figure \ref{fig:main_fig} \emph{(d)}.

We implemented a polynomial-time enumerator for the regions in a hyperplane 
arrangement that was further able to evaluate the verification LP on each such 
region. We used Python and a parallelism abstraction library, charm4py; all 
experiments were conducted on a system with a total of 24 Intel E5-2650 v4 
2.20GHz cores (48 virtual cores) of which our tool was allocated 24. The system 
had 256 GB of RAM.

\subsection{Region Enumeration} % (fold)
\label{sub:experiment_1_region_enumeration}
Figure \ref{fig:main_fig} \emph{(a)} shows the number of regions our tool found 
in $\mathcal{H\negthinspace S}(\tllparams{N}{M}^{(1)})$ for each TLL network, as 
well as the execution time required to enumerate them. For reference, the 
maximum number of regions possible is shown for each size, as determined by 
Theorem \ref{thm:arrangement_regions_bound}; note that $\mathcal{H\negthinspace 
S}(\tllparams{N}{M}^{(1)})$ has degeneracy in it, because its hyperplanes are 
differences between a common set of affine functions. Hence, 
$\mathcal{H\negthinspace S}(\tllparams{N}{M}^{(1)})$ has slightly fewer regions 
than the theoretical maximum that would be expected for a random arrangement.

Importantly, note the red curve, which is an exponential best-fit to the 
execution time of our tool. This shows that our region enumeration 
implementation is polynomial, and grows almost identically to the number of 
regions enumerated.
% subsection experiment_1_region_enumeration (end)

\subsection{Output Constraint Verification} % (fold)
\label{sub:output_constraint_verification}
For each example TLL network, we randomly generated an input constraint 
polytope and an output constraint to verify. Since $m=1$, an output constraint 
amounts to a random threshold, combined with a random choice of $\leq$ or 
$\geq$ to specify the constraint. A box-and-whisker plot of our tool's 
execution time on these verification problems is shown in Figure 
\ref{fig:main_fig} \emph{(b)}. The variability in execution time is due to the 
fact that our tool terminates early when a region of $\mathcal{H\negthinspace 
S}(\tllparams{N}{M}^{(1)})$ is found to generate a violation of the constraint.
% subsection output_constraint_verification (end)

\subsection{LTI System Forward Invariance Verification} % (fold)
\label{sub:lti_system_forward_invariance_verification}
For each of the TLL networks, we randomly generated LTI system matrices of the 
appropriate dimension for the TLL network to serve as a state-feedback 
controller. Then we used our tool to verify that each input polytope $P_x$ 
satisfies
\begin{equation}
	\forall x \in P_x ~.~ \left(A x + B \nn\subarg{\tllparams{N}{M}^{(1)})}(x)\right) \in P_x.
\end{equation}
That is $P_x$ is forward invariant for the system $(A,B)$ with closed-loop 
controller $\nn\subarg{\tllparams{N}{M}^{(1)})}$. This is easily accomplished 
since we consider only affine regions of $\tllparams{N}{M}^{(1)}$: the 
verification LP from before can be amended with a new objective function as 
follows, one for each of the $i = 1, \dots, \mathsf{N}_x$ constraints 
comprising $P_x$: $\max (\ell_{x,i}\circ(A + B\ell_0))(x)$.

The results of this experiment are shown in Figure \ref{fig:main_fig} 
\emph{(c)}. Note that the execution times are generally much better than in the 
previous experiment, despite the fact that twice as many ``output'' constraints 
are checked. This is ultimately because \emph{all} input polytopes were found 
not to be invariant, so the algorithm was almost always able to find a counter 
example early (as expected for randomized systems/controllers).
% subsection lti_system_forward_invariance_verification (end)

% section numerical_results (end)